\documentclass[conference]{IEEEtran}
\IEEEoverridecommandlockouts
\usepackage{amsmath,amssymb,amsfonts}
\usepackage{graphicx}
\usepackage{textcomp}
\usepackage{xcolor}

\usepackage{lipsum}  

\usepackage{algorithm}
\usepackage{algpseudocode}

\usepackage{amsthm}

\newtheorem{theorem}{Theorem}
\newtheorem{lemma}{Lemma}

\newtheorem{example}{Example}
\newtheorem{prop}{Proposition}

\newtheorem{definition}{Definition}
\newcommand{\A}{\mathcal{A}}
\newcommand{\X}{\mathcal{X}}
\newcommand{\Y}{\mathcal{Y}}

\newcommand{\HH}{\mathcal{H}}

\newcommand{\RR}{\mathbb{R}}

\newcommand{\TT}{\mathcal{T}}
\newcommand{\YY}{\mathcal{Y}}

\DeclareMathOperator{\opt}{opt}

\def\BibTeX{{\rm B\kern-.05em{\sc i\kern-.025em b}\kern-.08em
    T\kern-.1667em\lower.7ex\hbox{E}\kern-.125emX}}
\begin{document}

\title{ Finite Littlestone Dimension Implies\\ Finite Information Complexity
\\
\thanks{The work in this paper was supported in part by the Swiss National Science Foundation under Grant 200364.}
}

\author{\IEEEauthorblockN{Aditya Pradeep, Ido Nachum, and Michael Gastpar}
\IEEEauthorblockA{\textit{School of Computer and Communication Sciences} \\
\textit{EPFL, Lausanne, Switzerland}\\
\{aditya.pradeep, ido.nachum, michael.gastpar\}@epfl.ch}
}


\maketitle

\begin{abstract}
We prove that every online learnable class of functions of Littlestone dimension $d$ admits a learning algorithm with finite information complexity. Towards this end, we use the notion of a globally stable algorithm.  Generally, the information complexity of such a globally stable algorithm is large yet finite, roughly exponential in $d$. We also show there is room for improvement; for a canonical online learnable class, indicator functions of affine subspaces of dimension $d$, the information complexity  can be upper bounded logarithmically in $d$.

\end{abstract}


\section{Introduction}

Machine learning and information theory tasks are closely related since both require identifying patterns and regularities in data. This is why it is natural to formally study the following quantity:
\begin{align}
I(S;\A(S) ) ,
\end{align}
that is, the mutual information between the binary labeled training set $S=\left( (x_1,y_1),...,(x_m,y_m) \right) \in \left( \mathcal{X} \times \mathcal{Y} \right)^m $
and the output of a learning algorithm $ \A(S) \in \YY ^\mathcal{X}$. This quantity measures how many bits  a learning algorithm retains from $S$ to generate a prediction function over $\X$. Additionally, this quantity can be viewed as a measure of privacy, that is, the number of bits from $S$ the learning algorithm reveals.

In this work, we focus on online learnable~\cite{LD}  classes of functions $\HH \subset \{0,1\}^\X$ under finite mutual information $I(S;\A(S) ) $ constraints.  In the online learning setting, we sequentially  receive instances $(x_1,x_2,...)$. After receiving an instance $x_i$, we try to predict the correct label. Then, the true label $y_i$ is revealed.  Our performance is measured by the number of mistakes we make over the sequence. The Littlestone dimension of a hypotheses class $\HH$ measures the maximal number of mistakes the optimal algorithm makes over an arbitrary sequence  $S=\left( (x_1,f(x_1)),(x_2,f(x_2)),...  \right) $. Formally, the Littlestone dimension is defined to be 
\begin{align}\label{eq:ld}
   \opt(\HH)=  \min_{\A} \max_{ S }  \sum_{i=0}^\infty |\A( S[1:i] )(x_{i+1})-f(x_{i+1})| ,
\end{align}
where the infimum is over all algorithms (or mappings) $\A:  ( \X \times \{0,1\} )^*  \rightarrow \{0,1\}^\X  $, the supremum is over all sequences $S=\left( (x_1,f(x_1)),(x_2,f(x_2)),...  \right) $ where $x_i \in \X$, $f\in \HH$, and  $S[1:i]=\left( (x_1,f(x_1)),...,(x_i,f(x_i)) \right) $.


A canonical example for a class of infinite Littlestone dimension is the class of thresholds 
$\TT=\left\{ 1_{x\geq a}~:~a\in \RR \right\} \subset \{0,1\}^
       {\RR}$. For this class, under the probably approximately correct (PAC) learning~\cite{PAC} setting,  \cite{bassily2018learners} shows that for all $r\in \RR$ and for any empirical risk minimization (ERM) algorithm  there exists a distribution over $\RR$ and  a function $f\in \TT$ such that $I(S;\A(S))>r$. More so, \cite{livni2020limitation} shows that for  any algorithm with a true error guarantee smaller than $1/4$ there exists a distribution  over $\RR$ and  a function $f\in \TT$ such that   $I(S;\A(S))>r$.
       
       The above shows that any class of infinite Littlestone dimension cannot be learned under finite mutual information constraints. This holds since any class of infinite Littlestone dimension contains an infinite class of threshold functions (see Fact 5.1 in~\cite{model}).  Our main results, Theorem \ref{thm:main} and Theorem \ref{theorem-two}, show that the converse holds true as well: any class of finite Littlestone dimension can be learned under finite mutual information constraints. 
       
       Altogether, we now have an equivalence between online learning (combinatorial definition) and PAC learning under finite mutual information constraints (probabilistic definition). This is reminiscent to the  fundamental  theorem of statistical learning~\cite{VC-PAC} which  also draws an equivalence between the  combinatorial definition of the VC-dimension~\cite{VCoriginal1971} and  the probabilistic definition of PAC learning. 
       
 The derivation of Theorem~\ref{thm:main} relies on the notion of globally stable algorithms. We prove that any class that admits a globally stable algorithm is also learnable with finite mutual information constraints. The notion of globally stable algorithms was introduced in~\cite{ld->dp} to prove that any  finite Littlestone dimension class is  PAC learnable with  $(\epsilon,\delta)$-differential privacy (for any $\epsilon, \delta > 0$). And in a similar flavor to our result,~\cite{dp->ld} completes the equivalence between finite Littlestone dimension and  PAC learnability with  $(\epsilon,\delta)$-differential privacy. We remark that there is no direct link between $(\epsilon,\delta)$-differential privacy and a bound over the mutual information. For example, see Lemma 2.6 in~\cite{de2012lower}, where  an  $(\epsilon, \delta)$-differential private algorithm is given with $\delta \sim 2^{-n}$, where $n$ is the sample size, and the mutual  information between input and output is roughly $n$, so it is unbounded.


In this work, we upper bound the information complexity by roughly ${2^d}$ in the Littlestone dimension $d$. Therefore, we also show there is room for improvement. We specifically study the class of indicator functions of affine subspaces defined in Example \ref{Example-one}.  This is the most canonical example for finite Littlestone dimension (for other examples see  Section 5.1 in~\cite{model}).


 For this class, we present an algorithm which has information complexity of roughly $\log_2(d)$. This leaves us with the open question of how small is the information complexity for a general finite Littlestone dimension class.

The rest of the paper is organized as follows.  In Section \ref{sec:pre}, we provide all the necessary and exact definitions we work with in the paper. In Sections \ref{sec:main} and \ref{sec:proofs}, we state and prove our main theorems. Finally, in Section \ref{sec:Hdl}, we present an improved information complexity bound for indicator functions of affine subspaces.






\section{Related Work}\label{sec:rel}
  The mutual information between the input training set and the output can be thought of as a stability parameter of the algorithm. \cite{bousquet2002stability} proposed various generalization error bounds for learning algorithms by defining various notions of stability. \cite{russo2019much} and \cite{xu2017information} first proposed upper bounds for generalization error using mutual information and these results have been generalized to other measures of dependence between input and output  in  \cite{bu2020tightening}, \cite{steinke2020reasoning}, \cite{esposito2021generalization} and \cite{banerjee2021information}.

\cite{ghazi2021sample} presented an algorithm for finite Littlestone classes with improved privacy parameters, from doubly exponential in $d$ to polynomial in $d$. Their insight might be used to improve the information complexity bound in this work.

\section{Preliminaries}\label{sec:pre}

\subsection{Information Complexity of Learning}

\subsubsection*{PAC Learning}

We use the standard notation of PAC-Learning \cite{UML}. A hypothesis $h$ is a function  $h:\X \rightarrow \Y$ where $\X$ is the input set and $\Y = \{0,1\}$ is the possible values the function takes. An example is a pair $(x,y) \in \X \times \Y$ and a sample $S$ is a finite sequence of examples. The \emph{empirical error}  of a hypothesis $h$ with respect to $S$ is defined by 
\begin{align}
    L_S(h) = \frac{1}{|S|}\sum_{ (x_i,y_i )\in S} \mathbf{1}[h(x_i) \neq y_i].
\end{align}

The \emph{true error} of a hypothesis $h$ with respect to a distribution $D$ over $\X \times \Y$ is defined by 
\begin{align}
    L_D(h) = Pr_{(x,y)\sim D} [h(x) \neq y].
\end{align}

Let $\mathcal{H} \subset \Y^\X $ be a hypothesis class. 
 $S$ is said to be \emph{realizable} by $\mathcal{H}$ if there exists $h \in\mathcal{H}$ such that $L_S(h)=0$. $D$ is \emph{realizable} by $\mathcal{H}$ if there exists $h \in\mathcal{H}$ such that $L_D(h)=0$. 
A learning algorithm   $\A: (\X \times \Y)^* \rightarrow \Y^\X $ is a mapping (possibly randomized) from input samples to output hypotheses, where $(\X \times \Y)^* = \cup_{m=1}^\infty \left( \mathcal{X} \times \mathcal{Y} \right)^m$ is the set of all finite samples. $\A(S)$ denotes the distribution over hypotheses induced by the algorithm when the input sample is $S$. 

We say an algorithm $\A$  \emph{PAC learns} a hypothesis class  $\mathcal{H}\subset \Y^\X$  if there exists  a function $m:(0,1)\times (0,1) \rightarrow  \mathbb{N}$ such that for all $0<\epsilon,\delta<1$, $h\in \HH$, and any realizable  distribution  $\X \sim D$ it holds with probability of at least $1-\delta$  (over the distribution $D$ and the randomness of $\A$),  $ L_D(\A(S)) <  \epsilon $ where $S$ is sampled i.i.d. from $D$ and $|S| \geq m(\epsilon,\delta)$.


\vspace{9pt}

\subsubsection*{Information Complexity}  \label{def:IC}

The information complexity of a hypothesis class $\HH$ is
\begin{align}
    IC(\HH) &= \sup_{|S|} \inf_{\A} \sup_{ D } I(S; \A(S)),
\end{align}
where the supremum is over all sample sizes $|S|\in \mathbb{N}$ and the infimum is over all  learning algorithms that PAC learn $\HH$.

\vspace{9pt}

Intuitively, the information complexity of a hypothesis class measures the maximal number of bits an optimum learning algorithm could reveal about the input sample.

We find it useful to define $S^k=(S_{1},S_2,\dots,S_k)$ are $k$ independent samples each of size $n$ (i.e)  $ S_i = \big( (x_{i1},y_{i1}),(x_{i1},y_{i1}),\dots,(x_{i n},y_{i n}) \big) $.


\subsection{Online Learning and the Littlestone dimension}\label{Sec-prelim-littlestone}

 \label{def:ls_dim}

An equivalent definition for the    \emph{Littlestone dimension} (compared to Equation~\eqref{eq:ld} arises from \textit{mistake trees}~\cite{LD} - binary decision trees where every node is an element of $\X$. A root-to-leaf path in a mistake tree is a sequence of examples $(x_1,y_1),\dots,(x_d,y_d)$ where $x_i$ is the $i$'th node in the binary tree and $y_i=1$ if the next $(i+1)$'th node is the right child of the $i$'th node and $y_i=0$ if it is the left node. 

A binary tree  is shattered by $\mathcal{H}$ if for every root-to-leaf path  $(x_1,y_1),\dots,(x_d,y_d)$ , $\exists h \in \mathcal{H}$ such that $h(x_i)=y_i$. The Littlestone dimension $d= opt(\mathcal{H}$) is the depth of the largest complete tree shattered by $\mathcal{H}$.

Intuitively, if a hypothesis class $\mathcal{H}$ has a Littlestone dimension $d$, then for any algorithm there exists a realizable sample $S$ such that the algorithm makes at least $d$ mistakes. Also, there exists an optimal algorithm that makes at most $d$ mistakes over any realizable sample. An example of such an algorithm is the \emph{standard optimal algorithm} ($SOA$)~\cite{LD}. 

\begin{example}[Indicator functions of affine subspaces]\label{Example-one}
We consider the class of indicator functions of affine subspaces of dimension at most $d$ in $\mathbb{R}^\ell$
\begin{align}
    \mathcal{H}_{d,l}= \{1_{w_1\cdot x=b_1}(x) \cdot... \cdot  1_{w_k\cdot x=b_k}(x) ~:\\ ~ w_i \in \RR^\ell, ~ b_i \in \RR, \nonumber
    ~k \leq \ell-d  \}
\end{align}
\end{example}

\subsection{An Online Learnable Class Admits a Globally Stable Algorithm}  

Theorem \ref{thm:main}  and Theorem \ref{theorem-two} utilize the notion of a globally stable algorithm and the fact that any finite Littlestone dimension class admits such an algorithm (Lemma~\ref{Lemma 1}).

\begin{definition}
An algorithm $\A$ is $(n,\eta)$-globally stable with respect to a distribution $D$ if there exists an hypothesis $h$ such that 
\begin{align}
    Pr_{S \sim D^n} [\A(S)=h] \geq \eta 
\end{align}
\end{definition}

\begin{lemma} \label{Lemma 1}
For a hypothesis class $\mathcal{H}$ with Littlestone dimension $d$, there exist an algorithm $G$ such that for all realizable distributions it holds that there exists a function $f$ for  $n=2^{2^{d+2} + 1}4^{d+1} \lceil {\frac{2^{d+2}}{\epsilon}} \rceil$, such that
$ Pr(G(S)=f) \geq  \eta=\frac{2^{-({2^d+1}) }}{d+1}$  and $L_D(f) < \epsilon$.
\end{lemma}

\noindent This follows from Theorem 10 in \cite{ld->dp}. They present an algorithm $G$ which is $(n,\eta)$-globally stable with respect to any realizable distribution.

\section{Finite Littlestone Dimension Implies Finite Information Complexity}\label{sec:main} 
The main result of this work shows that the mutual information between the training set $S$ and the output of a particular algorithm $A_G$ is finite.

Algorithm $A_G$ requires a possibly randomized $(n,\eta)$-globally stable  algorithm $G$  over distribution $\mathcal{D}$. This algorithm $G$ is run $k$ times on independent samples $S_1, S_2,\dots, S_k$. where each $S_i$ has $n$ samples each drawn i.i.d. from the distribution $D^n$. The outputs of these $k$ runs would then be $g_1,g_2,\dots g_k$ where $g_i=G(S_i)$. The Algorithm $A_G$ then outputs the most frequent function $g_{maj}$ if it occurs at least $\eta k/2\geq 2$ times. If no function occurs more than $\eta k/2$ times, then the algorithm returns Failure. To track the frequency of occurrence of each of the $k$ outputs of Algorithm $G$, we define the frequency function $C_k(f)=| \{ i :  g_i=f \} |$.\

\begin{algorithm}
\begin{algorithmic}
\caption{Algorithm $A_G$  }\label{alg:alg1} 
\Require $(n,\eta)$-globally stable algorithm $G$ 
\State $i \gets 1$
\State $j \gets 1$
\While{$i \leq k$} 
  \State $S_i = \big( (x_{i1},y_{i1}),(x_{i1},y_{i1}),\dots,(x_{i n},y_{i n}) \big) $
  \State Run $G$ on $S_i$ and generate function $g_i$  
  \State $i \gets i+1$
\EndWhile

\While{$j \leq k$} 
\State $c_j \gets C_k(g_j)$
  \State $j \gets j+1$
\EndWhile

\If{$\max_{1\leq i \leq k}(c_i) \geq \frac{\eta k}{2}$}
 \State Output $g_{maj} : C_k(g_{maj})=\max_{1\leq i \leq k}(c_i)$ 
 \Else{ Output Failure}
\EndIf
\end{algorithmic}

\end{algorithm}

\begin{theorem}\label{thm:main}  

Let $G$  be a $(n,\eta)$-globally stable algorithm with respect to a distribution $\X \times \Y \sim D$. Then, $A_G$ satisfies the following:

\begin{enumerate}
    \item Algorithm $A_G$   satisfies
\begin{align}
    I(S^k;A_G(S^k))  \leq 2^{3+\log_2(k)-\frac{\eta k}{2}-k\log_2(1-\frac{\eta}{2})}& \nonumber\\
    + \log_2(\frac{4}{\eta}) +\frac{3}{e\ln(2)}&.
\end{align}
    \item Algorithm $A_G$ outputs a  function with probability of at least $1 - \delta $ where $\delta < e^{-k\eta^2/2}$.
\end{enumerate}
\end{theorem}

The proof of this theorem is given in Section~\ref{sec-proof-thm:main}.

\begin{theorem}\label{theorem-two}
The information complexity of every class $\mathcal{H}$ of  Littlestone dimension $d$ is bounded by 
\begin{align}
    IC(\mathcal{H}) \leq  2^d+\log_2(d+1)+3+\frac{3}{e\ln(2)}. \label{eq:thm2}
\end{align}
\end{theorem}

This theorem shows that finite Littlestone dimension implies finite information complexity. The proof of this theorem is given in Section~\ref{sec-proof-theorem-two}.
Specifically, we leverage Lemma \ref{Lemma 1} that there exists a globally stable algorithm which exists for any hypothesis class of finite Littlestone dimension. Using this and Theorem \ref{thm:main}, the above result follows. 

\section{Proofs}\label{sec:proofs}
\subsection{Proof of Theorem \ref{thm:main}}\label{sec-proof-thm:main}

\begin{proof}

We show that the mutual information between the input and output of Algorithm $A_G$ is finite. Let $\mathcal{F}$ be the set of possible outputs of Algorithm $A_G$. We partition this set $\mathcal{F}$ into the the following 3 sets: 
\begin{enumerate}

    \item  Failure.
    \item $\mathcal{F}_1$: the set of functions $f_1 \in \mathcal{F}_1$ such that
    \begin{align}
   Pr (G(S)=f_1) \leq \eta/4 . \label{eq:F1}
    \end{align}
    \item $\mathcal{F}_2$: the set of  functions  $f_2 \in \mathcal{F}_2$ such that
    \begin{align}
    \eta/4 < Pr(G(S)=f_2). \label{eq:F2}
    \end{align}
\end{enumerate}

In our proof, we find it convenient to use the random variable $C_k(f)$ 
which counts  how many times  Algorithm $G$ outputs the function $f$ after $k$ runs.
Note that this random variable is induced by the randomness in the selection of the set $S$ as well as  the randomness in the algorithm $G(\cdot).$ We also define $P_f =  Pr(A_G(S^k)=f )$. \\

The total mutual information is bounded by the entropy of the output of the algorithm and this is calculated by simply expanding the definition of the entropy as follows
\begin{align}
    &I(S^k;A_G(S^k)) \leq  H(A_G(S^k)) = \sum_{f \in \mathcal{F}} \left(- P_f \log_2(P_f)\right)\\ 
    &= \underbrace{- Pr(\text{Failure}) \log_2(Pr(\text{Failure}))}_{=H_{failure}} + \nonumber \\& \,\,\,\,\,\,\,\,\,\,\,\, + \underbrace{\sum_{f \in \mathcal{F}_1} \left(- P_f \log_2(P_f)\right)}_{=H_1} +  \underbrace{\sum_{f \in \mathcal{F}_2} \left(- P_f \log_2(P_f)\right)}_{=H_2} 
\end{align}

We bound each of the 3 terms in the  entropy calculation.\\

\textit{Probability of Failure.} 
The probability of failure of Algorithm $A_G$ is the probability with which no function occurs more than $\eta k /2$ times in the $k$ runs of algorithm $G$. Note that $G$ is a $(n,\eta)$-globally stable algorithm for the distribution $\mathcal{D}$ and thus there exists a function $f_0$ which $G$ outputs with probability of at least $\eta$.
\begin{align}
    Pr(\text{Failure}) &=Pr( C_{k}(f_0) < \eta k/2) \\
    &\leq Pr( C_{k}(f_0) \leq (\eta - \eta/2)k) \leq e^{-k\eta^2/2}.\label{eq:failure}
\end{align} 
The last line follows from Hoeffding's inequality \cite{hoeffding1994probability}. Also,
\begin{align}
    H_{failure} =-Pr(\text{Failure})\log_2 Pr(\text{Failure}) \leq \frac{1}{eln(2)}. \label{eq:Hfailure}
\end{align} 

since $H_{failure}$ achieves maxima at $Pr(\text{Failure})= \frac{1}{e}$.\\

\noindent \textit{Upper bound  of $H_1$.}  
To bound $H_1$, we consider it useful to split $\mathcal{F}_{1}$  into the following subsets $\mathcal{F}_{1 ,\lambda}$ with $\lambda =2,3,\dots$
\begin{align}
 \mathcal{F}_{1, \lambda}  &= \left\{f : \frac{\eta}{2^{\lambda+1}} < Pr(G(S) = f)\leq \frac{\eta}{2^{\lambda} }\right\}  \label{eq:case2fn}
\end{align}
Note that by definition in Equation~\eqref{eq:case2fn},  $ |\mathcal{F}_{1,\lambda}| \leq \frac{2^{\lambda+1} }{\eta}$ as the total probability is bounded by 1. Let us bound the probability $P_f$ for $f \in \mathcal{F}_{1, \lambda}$. Define $q_f=Pr(G(S)=f)$. We use the following notion of KL divergence $D_{KL}(P || Q) = \sum_{x \in \mathcal{X}}P(x)\ln(P(x)/Q(x))$
\begin{align}
    P_f &= Pr (C_k(f) \geq \eta k /2 )  \\
    & = Pr(C_k(f) \geq k( q_f + \eta/2 -q_f) ) \\
    &\leq  e^{-k D_{KL}( \eta/2|| q_f)} \label{eq:eqn_usechernoff} \\
    &= e^{-\frac{k \eta}{2}\ln(\frac{\eta}{2q_f})} e^{-k(1-\frac{ \eta}{2})\ln(\frac{1-\eta/2}{1-q_f})}  \\
    &\leq 2^{-\frac{\eta k}{2}(\lambda-1)}(1-\small{\frac{\eta}{2}})^{-k}   \label{eq:eqchernoff2}
\end{align}
Equation~\eqref{eq:eqn_usechernoff} follows from Chernoff-Hoeffding theorem (Equation 2.1 in \cite{hoeffding1994probability}) and since $\eta/2^{\lambda+1} < q_f \leq \eta/2^{\lambda}$.
\begin{align}
    H_1 &= \sum_{f \in \mathcal{F}_1} -P_f \log_2(P_f) \\
    &= \sum_{\lambda=2}^\infty \sum_{f \in \mathcal{F}_{1,\lambda}}  -P_f \log_2(P_f) \\
    & \leq  \sum_{\lambda=2}^\infty  |\mathcal{F}_{1,\lambda}| ( -P_f \log_2(P_f)) \leq  \sum_{\lambda=2}^\infty \frac{2^{\lambda+1} }{\eta}( -P_f \log_2(P_f) ) \\
    &\leq \frac{2}{e\ln(2)} + \sum_{\lambda=2}^\infty   \frac{2^{\lambda+1} }{\eta} 2^{-\frac{\eta k}{2}(\lambda-1)}(1-\small{\frac{\eta}{2}})^{-k} \frac{\eta k (\lambda -1 ) }{2} \label{eq:equation 2e}\\
    &=  \frac{2}{e\ln(2)} + 2k \sum_{\lambda=1}^\infty \lambda r^\lambda(1-\small{\frac{\eta}{2}})^{-k}\\
    &=  \frac{2}{e\ln(2)} + \frac{2kr(1-\small{\frac{\eta}{2}})^{-k}}{1-r^2}. \label{eq:H-one}
\end{align}
where $r=2^{1-\frac{\eta k}{2}}$. 
Equation~\eqref{eq:equation 2e} follows from the upper bound of $P_f$ in Equation~\eqref{eq:eqchernoff2}. However, since the function $H(x)=-x\log_2(x)$ is a concave function, we cannot normally directly substitute the upper bound for $P_f$. However we can do so if $P_f \leq \frac{1}{e}$ in Equation~\eqref{eq:equation 2e} since $H(x)$ is an increasing function for $0 \leq x \leq \frac{1}{e}$ and decreasing for $ \frac{1}{e} \leq x \leq 1$ . There are at most 2 functions $f$ which have $P_f>\frac{1}{e}$ and hence the additive correction term in Equation~\eqref{eq:equation 2e}.\\

\noindent \textit{Upper bound of $H_2$.} 
\noindent For all functions $f_2 \in \mathcal{F}_2$, we have
\begin{align}
    \eta/4 < Pr [G(S)=f_2] .
\end{align}
\noindent  Since the total probability cannot exceed 1, $|\mathcal{F}_2| < \frac{4}{\eta}$. Hence,  
\begin{align}
H_2 \leq \log_2(|\mathcal{F}_2|) \leq\log_2( \frac{4}{\eta} ). \label{eq:H-two}
\end{align}
\noindent  Thus, the mutual information $I(S^k;A_G(S^k))$ is bounded by 
\begin{align}\
  H_{failure} + H_1 + H_2 \leq \frac{3}{e\ln(2)}+ \frac{2kr(1-\small{\frac{\eta}{2}})^{-k}}{1-r^2} + \log_2( \frac{4}{\eta} ) &\\
   \leq \frac{3}{e\ln(2)}+ 4kr(1-\small{\frac{\eta}{2}})^{-k} + \log_2( \frac{4}{\eta} ) \label{eq:MIterms}&\\
  = 2^{3+\log_2 k -\frac{\eta k}{2}-k\log_2(1-\frac{\eta}{2})} + \log_2( \frac{4}{\eta}) +\frac{3}{e\ln(2)}.&
\end{align}
Equation~\eqref{eq:MIterms} follows since $\frac{\eta k}{2} \geq 2$ and $r=2^{1-\frac{\eta k}{2}} \leq \frac{1}{2} \implies 1-r^2 > 1/2$.
\end{proof}
\textbf{Remark.} While the output of Algorithm G may be a mixture over a discrete and continuous distribution, Algorithm $A_G$ always outputs functions over a discrete distribution. To see why, observe that the event where any function from the continuous support  appears more than once when running Algorithm $A_G$ happens with probability zero. This line of reasoning  makes the entropy of the output of $A_G$ well defined.

\subsection{Proof of Theorem \ref{theorem-two}}\label{sec-proof-theorem-two}

To prove Theorem \ref{theorem-two}, we prove the following lemmas.

\begin{lemma} \label{Lemma 2}
Let $G$ be a learning algorithm that satisfies $G(S)=G( S[1:n_1] )$ for all $S$ and for any $f$ such that $Pr(G(S)=f) \geq \eta$ and $L_{ S[1:n_1] }( f )=0$ for all $S$ such that $G(S)=f$. Then,  
\begin{align}
     L_D(f)  \leq \frac{\log_2(\frac{1}{\eta})}{n_1}.
 \end{align}
 \end{lemma}


\begin{proof}

Let $L_D(f) = \epsilon $ and $E_1$ be the event that the algorithm $G$ outputs $f$. By definition, $Pr(E_1) \geq \eta$. Let $E_2$ be the event that algorithm $G$ is consistent\footnote{A consistent learner $\A$ for a hypothesis class $\mathcal{H}$ satisfies $L_S(A(S))=0$ for every input sample $S$.} with the first $n_1$ samples of the input. $Pr(E_2) \leq (1-\epsilon )^{n_1}$. As $G$ is  consistent over the first $n_1$ samples, $E_1 \subseteq E_2$ and hence $\eta \leq (1-\epsilon )^{n_1}\leq 2^{-n_1\epsilon }$.
\end{proof}


\begin{lemma} \label{Lemma 3}
$A_G$  outputs function $f$ such that $ Pr [G(S)=f] > \eta/4  $ with probability at least $1-\delta$. Hence, $A_G$ PAC learns $\mathcal{H}$ with $m( 2 \epsilon, \delta) =2^{2^{d+3}}4^{d+1}(d+1) \lceil {\frac{2^{d+2}}{\epsilon}} \rceil \max(4\ln(1/\delta) ,10 )$.
\end{lemma}   

\begin{proof}
Recall that $\mathcal{F}_2$ is the set of  functions  $f \in \mathcal{F}_2$ such that $  Pr (G(S)=f) > \eta/4  $

\begin{align}
    \sum_{f\in \mathcal{F}_2} P_f &= 1 - Pr(Failure) - \sum_{f\in \mathcal{F}_1} P_f \\
    &\geq 1 -  e^{-k\eta^2/2} - \frac{2^{-\eta k /2}}{1- 2^{-\eta k /2}} \\
    &\geq 1-\delta \label{eq:delta}
\end{align}

This follows from Equations~\eqref{eq:failure} and~\eqref{eq:eqchernoff2}. Equation~\eqref{eq:delta} follows by choosing $k \geq 4\ln(1/\delta)/\eta$. We set $k = \max ( \frac{4\ln(\frac{1}{\delta})}{\eta},\frac{10}{\eta} ) $.

Algorithm $G$ takes as input a sample $S$ of size $n$ but will not use the whole sample. However, it will use at least $n_1$ input samples. From Lemma \ref{Lemma 2} with $n=2^{2^{d+2} + 1}4^{d+1} \lceil {\frac{2^{d+2}}{\epsilon}} \rceil$, $n_1=\lceil {\frac{2^{d+2}}{\epsilon}} \rceil $ and $\eta=\frac{2^{-({2^d+1}) }}{d+1}$, we have that the true loss of any hypothesis output by $G$ is upper bounded by $\epsilon$. 

The true loss of any function $f \in \mathcal{F}_2$ which is output by $A_G$ is bounded as
\begin{align}
    L_D(f) \leq \frac{\log_2(\frac{4}{\eta})}{n_1} < 2 \epsilon.
\end{align}
Thus algorithm $A_G$ PAC learns $\mathcal{H}$  with $2\epsilon$ true error, $(1-\delta)$ confidence and sample complexity $m(2  \epsilon,\delta) =  nk$. The total sample complexity follows since Algorithm $A_G$ runs Algorithm $G$ $k$ times.
\end{proof}

\noindent Now we prove Theorem \ref{theorem-two}. From Theorem \ref{thm:main} we have an upper bound for $I(S^k;A_G(S^k))$  for any distribution $D$ over a class ${\mathcal H}$ with Littlestone dimension $d$. Thus, by our definition of information complexity,
\begin{align}
   IC(\mathcal{H}) \leq 2^{3+\log_2(k)-\frac{\eta k}{2}-k\log_2(1-\frac{\eta}{2})}
   + \log_2(\frac{4}{\eta})  +\frac{3}{e\ln(2)}. \label{eq:corr1}
\end{align}
The above bound follows as $A_G$ PAC learns $\mathcal{H}$ from Lemma \ref{Lemma 3}. Notice that this bound is decreasing in $k$ as $\frac{\eta}{2}+\log_2(1-\frac{\eta}{2})>0$ for $0<\eta\leq1$. As $k \rightarrow \infty$, the first term in Equation~\eqref{eq:corr1} tends to $0$. Using the globally stable algorithm from Lemma \ref{Lemma 1}, we have that the upper bound for information complexity grows with $2^{d}$.
\begin{align}
    IC(\mathcal{H}) \leq 2^d+\log_2(d+1)+3+\frac{3}{e\ln(2)}.
\end{align}

\noindent This upper bound could be improved by a different globally stable algorithm. 
\\

\section{An improved Information Complexity Bound for $\mathcal{H}_{d,l}$ }\label{sec:Hdl}

The class of indicator functions of affine subspaces $\mathcal{H}_{d,l}$ has finite Littlestone dimension $d+1$, see Example~\ref{Example-one} above. For this class, we suggest a modification for  the globally stable algorithm proposed in \cite{ld->dp}, thus, we can improve the  information complexity bound from  roughly  $2^{d}$ to roughly $\log_2(d)$. 
\begin{prop}
 $ IC(\mathcal{H}_{d,l}) \leq \log_2(d+1)+2+\frac{3}{e\ln(2)} $
\end{prop}

\noindent {\it Proof Outline:} We  provide a proof sketch for the proposition. The rough  idea in~\cite{ld->dp} is to iteratively  compare the output of the $SOA$ (see Section~\ref{Sec-prelim-littlestone}), namely, $f_1=SOA(T_1)$ and $f_2=SOA(T_2)$ over increasingly larger independent samples $T_1$ and $T_2$. We can  consider two cases: 
\begin{enumerate}
\item $ f_1 = f_2$. We output  $f_1$ as this likely occurs only if $Pr(f_1=SOA(T_1))$ is large which corresponds to global stability. 
    \item $ f_1 \neq f_2$. Then, there exists $x$ such that $f_1(x)\neq f_2(x)$. This is why the $SOA$ will make an extra mistake over both samples $(T_1,(x,f_2(x)))$ or $(T_2,(x,f_1(x)))$. Since we don't know the true label of $x$, we randomly choose  $(T_1,(x,f_2(x)))$ or $(T_2,(x,f_1(x)))$.
\end{enumerate}

Without getting into more specific details, as long as $ f_1 \neq f_2$, the procedure above can be iterated many times to generate  a larger and larger sample for which, in the end, the $SOA$ will make $d$ mistakes. To reach such a sample we  make at most $2^d$ random label choices  (when we follow the proof in \cite{ld->dp}). Thus, if the choices we make are consistent with  the true hypotheses $h$, then  $h$ will be our output  with probability of approximately
\begin{align}\label{eq:eta1}
    \eta \approx \frac{1}{2^{2^{d}}} \cdot \frac{1}{d+1}, 
\end{align}

where the first factor corresponds to correct random choices and the second factor to other components of $G$. This means $G$ is $(n,\eta)$-globally stable (for large enough $n$)  and we  upper bound  the information complexity  with roughly $2^{2^d}$.

To have a better bound for the information complexity we point out  a key property  of $\mathcal{H}_{d,l}$.  The output of the $SOA$ for a realizable sample $S$ is the hyperplane defined by the positively labeled instances, i.e., the hyperplane of minimal dimension that passes through these instances.

Now,  when we consider item 2 above for $\mathcal{H}_{d,l}$, we note that on any instance $x$ such that $f_1(x)\neq f_2(x)$ the true label of $x$ is $1$.  This is true because the true hyperplane $h$ must agree with the positive labels of both $f_1$ and $f_2$ since it passes through the positively labeled instances of $T_1$ and $T_2$. This is why in item 2 we no longer need to randomly choose  between $(T_1,(x,f_2(x)))$ or $(T_2,(x,f_1(x)))$. We now output $h$ with probability of roughly $1/(d+1)$  since now the first factor, compared to Equation~\eqref{eq:eta1}, is irrelevant. So, we have that this modified version of $G$ is $(n,1/(d+1))$-globally stable so  we can upper bound  the information complexity with roughly $d$.

\newpage
\bibliographystyle{IEEEtran}
\bibliography{ldbib}
\end{document}